\documentclass[conference]{IEEEtran}
\IEEEoverridecommandlockouts
\usepackage[numbers]{natbib}
\usepackage{bbm}
\usepackage{url}
\usepackage{balance}
\usepackage[hidelinks]{hyperref}
\hypersetup{
    colorlinks=true,
    citecolor=blue,
    linkcolor=blue,
    filecolor=magenta,      
    urlcolor=cyan
}

\usepackage{amsmath,amssymb,amsfonts,amsthm}
\usepackage{changes}
\newtheorem{lemma}{Lemma}
\usepackage{caption}
\usepackage{algorithmic}
\usepackage{graphicx}
\usepackage{textcomp}
\usepackage{xcolor}
\usepackage{cleveref}
\def\BibTeX{{\rm B\kern-.05em{\sc i\kern-.025em b}\kern-.08em
    T\kern-.1667em\lower.7ex\hbox{E}\kern-.125emX}}
\usepackage[most]{tcolorbox}

\makeatletter
\def\@IEEEsectpunct{~~}
\renewcommand\paragraph{\@startsection{paragraph}{4}{\z@}%
  {-1.25ex \@plus -1ex \@minus -0.2ex}%
  {0pt}%
  {\bfseries}%
}
\newcommand{\maybe@addperiod}[1]{%
  #1\@addpunct{.}%
}
\makeatother
\usepackage{xspace}
\newcommand{\dataproof}{training data proof\xspace}
\newcommand{\dataproofs}{training data proofs\xspace}
\newcommand{\Dataproofs}{Training data proofs\xspace}

\definecolor{mygreen}{RGB}{0, 128, 0}

\begin{document}

\title{Position: Membership Inference Attacks Cannot Prove that a Model Was Trained On Your Data}

\newboolean{arxiv}
\setboolean{arxiv}{true}
\ifthenelse{\boolean{arxiv}}{
  \author{
    Jie Zhang$^1$\quad Debeshee Das$^1$ \quad Gautam Kamath$^{2}$ \quad Florian Tramèr$^1$ \\ \\
    { $^1$ETH Zurich \quad $^2$University of Waterloo}
  }
}{
}

\maketitle

\begin{abstract}
We consider the problem of a \emph{\dataproof}, where a data creator or owner wants to demonstrate to a third party that some machine learning model was trained on their data. \Dataproofs 
play a key role in recent lawsuits against foundation models trained on Web-scale data.
Many prior works suggest to instantiate \dataproofs using \emph{membership inference attacks}. We argue that this approach is statistically unsound: to provide convincing evidence, the data creator needs to demonstrate that their attack has a low false positive rate, i.e., that the attack's output is unlikely under the \emph{null hypothesis} that the model was \emph{not} trained on the target data. Yet, sampling from this null hypothesis is impossible, as we do not know the exact contents of the training set, nor can we (efficiently) retrain a large foundation model.
We conclude by offering three paths forward, by showing that membership inference on special canary data, watermarked training data, and data extraction attacks can be used to create sound \dataproofs.

\end{abstract}


\section{Introduction}
Large generative models trained on vast amounts of internet data, such as GPT-4~\cite{achiam2023gpt}, DALL-E~\cite{openai_dalle3}, or LLaMa~\cite{touvron2023llama}
have sparked a debate about copyright infringement~\cite{wei2024evaluating, mainidi2024}. 
A number of recent lawsuits allege that non-permissive data was used to train these models, e.g, in Getty Images v.\ Stability AI~\cite{GettyImages_v_StabilityAI}, Doe v.\ Github~\cite{doe_v_github_2024}, or New York Times v.\ OpenAI~\cite{nyt_v_OpenAI}.

These debates call for methods that let a data creator or owner provide evidence (e.g., to a judge) that their data was used to train some machine learning (ML) model. We call this a \dataproof (see~\Cref{fig:intro}). This problem is closely related to \emph{membership inference attacks}, a form of privacy attack that tests whether a given data point was used to train a model~\cite{HomerSRDTMPSNC08,shokri2017membership}.
Unsurprisingly then, a growing body of work proposes to re-purpose membership inference (MI) attacks to create \dataproofs~\cite{shi2023detecting, meeus2024did, meeus2024inherent, eichler2024nob, zhang2024min, duarte2024cop, meeus2024copyright} (this has also been referred to as a \emph{dataset inference attack}~\cite{maini2021dataset}).
The core principle of these techniques is to show that the target model exhibits some behavior (e.g., a low loss on some data) that would be highly unlikely if the model had not been trained on that data. 

We argue that membership inference attacks are fundamentally flawed for this task and cannot provide convincing evidence of the use of (arbitrary) data in training. 
Our position is summarized as follows:

\begin{tcolorbox}
\textbf{Position:} Membership inference attacks \textbf{cannot} give reliable evidence of data usage in the training of production ML models, because it is impossible to bound the attack's false positive rate (FPR).
\end{tcolorbox}

The core issue with using an MI attack to construct a \dataproof is the need to estimate a model's behavior in a \emph{counterfactual} scenario, where it was not trained on the target data. This is necessary in order to show that the model's observed behavior is ``unusual''.
We formalize this notion (which is often left implicit in the literature) by drawing on the standard interpretation of a membership inference attack as a hypothesis test~\cite{yeom2018privacy}.
We can then recast the above counterfactual estimation problem as that of estimating the model's behavior \emph{under the null hypothesis} that the target data was not present in the model's training data.

Yet, this estimation is \emph{impossible} for large-scale proprietary models. 
Since the model's full training set is unknown and exact training procedures are opaque (and prohibitively expensive), we cannot simply train a new model without the target data. In summary, we cannot \emph{sample from the null hypothesis}, and thus cannot bound the attack's \emph{false positive rate} (the probability that the data creator would wrongly accuse the model developer of using their data for training).

\begin{figure}[t]
    \centering
    \includegraphics[width=\linewidth]{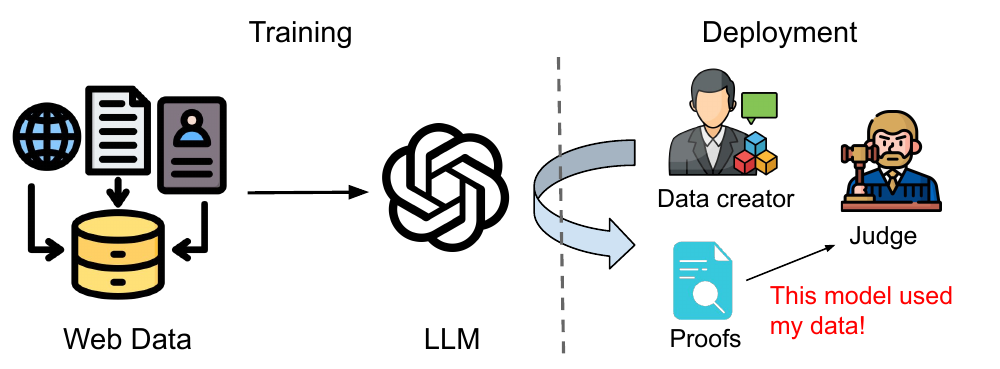} 
    \caption{In a \emph{\dataproof}, a data creator aims to convince a third party
    (e.g., a judge) that a machine learning model was trained on their data. A common proposal in the literature is to use membership inference attacks for this purpose.}
    \label{fig:intro}
\end{figure}

Existing work~\cite{meeus2024did, meeus2024inherent, shi2023detecting, eichler2024nob, duarte2024cop, mainidi2024, zhang2024min} proposes various heuristics to estimate an attack's 
false positive rate, but we argue that all such approaches are statistically unsound, providing plausible deniability for model developers:
\begin{enumerate}
    \item Some works~\cite{shi2023detecting, zhang2024min, meeus2024did, duarte2024cop} estimate the FPR on data that is known not to be in the training set (e.g., data created after the model was trained).
    However, prior work~\cite{meeus2024inherent, das2024blind, mainidi2024, duan2024membership} shows that this introduces distribution shifts between the training members and non-members, which leads to unreliable evaluations of MI attacks.
    Here, we build on these works and argue that this issue makes it impossible for \emph{any} MI attack to give convincing evidence for a \dataproof.
    
    \item As a response, recent works have suggested to use ``debiased''
    sets of non-members to evaluate an attack's FPR~\cite{meeus2024inherent, eichler2024nob}. We argue that while this approach can alleviate the distribution shift issue, it results in a \dataproof with very low power and applicability.

    \item Finally, a line of work on \emph{dataset inference}~\cite{maini2021dataset, mainidi2024} suggests evaluating an attack's FPR using \emph{held-out counterfactuals} of a piece of data (e.g., unpublished drafts of a news article or book).
    We show this approach is also unsound because it ignores the \emph{causal} effect that publishing data may have on other samples in the model's training set. That is, even if a model is \emph{not} trained on some particular piece of target data being tested, the mere act of publishing this target data may influence other content that the model was trained on. 
\end{enumerate}

Essentially, we argue that MI attacks (both currently existing and hypothetical future ones) cannot hope to provide convincing evidence of training data usage.

We then show that three alternative paradigms can yield convincing \dataproofs.
The first paradigm borrows from the literature on privacy auditing~\cite{carlini2019secret,JagielskiUO20,NasrSTPC21, steinke2023privacy} and injects randomly sampled \emph{canaries} into the data, and then tests how the model \emph{ranks} the selected canary among all alternatives. We formalize 
this idea and highlight the conditions under which it provides a sound \dataproof.

In the second paradigm, data owners can add watermarks to their data before releasing it publicly~\cite{gloaguen2024discovering,zhong2024watermarking}. This watermarking process allows them to later test whether a model has been trained using their watermarked data. The detection process works as a hypothesis test that can identify unauthorized data usage with high confidence (low false positive rate)~\cite{sander2025watermarking}.

The third paradigm uses stronger \emph{data extraction attacks}~\cite{carlini2021extracting}, to show that we can recover non-trivial portions of the data from the trained model. In this way, we bypass the need to estimate the counterfactual, as we assume that long-form data regurgitation would happen with \emph{negligible chance} if the model was not trained on that data (this is the informal argument made in some of the aforementioned lawsuits, which we make more rigorous in this work).

\section{\Dataproofs via Hypothesis Testing}
\label{sec:hypothesis}

We consider the case of a data creator or owner $C$ who wants to show that some target data%
\footnote{We make no assumption on the ``granularity'' of the target data here. This could be an individual piece of text or an image, or an entire collection of texts and/or images.} $x$ was used by a model developer $M$ to train some model $f$.

The data creator thus wants to first perform \emph{membership inference} to test whether $x$ was part of the model's training data, and then provide \emph{convincing evidence} to a third party that this inference is likely correct. 

It is natural to model this problem as a \emph{hypothesis testing problem}. Indeed, this is the approach taken either implicitly or explicitly in many prior works on membership inference (not necessarily with the goal of building \dataproofs)~\cite{mainidi2024, yeom2018privacy, carlini2022membership, meeus2024copyright}.
In these works, the \emph{null hypothesis} and \emph{alternative hypothesis} are  informally defined as follows:
\begin{align*}
    H_0&: \text{``the data $x$ was not in the training set of the model $f$''}\\
    H_1&: \text{``the data $x$ was in the training set of the model $f$''}
\end{align*}

The goal of the attack is then to \emph{reject the null hypothesis}, i.e., to show that the observed model behavior is unlikely to occur if $H_0$ were true. 
As we first argue, formalizing what these hypotheses mean in practical settings is actually non-trivial, and often overlooked in the literature.

\paragraph{What is the null hypothesis exactly?}
Most research papers on membership inference attacks focus on what we call ``in-the-lab'' settings, wherein researchers control the entire data generation process. That is, researchers train a model themselves and know what is part of the training set and what is not.
In this case, the implicit meaning behind $H_0$ and $H_1$ is clear: the training dataset $D$ is randomly sampled from some larger data universe $U$ (e.g., the CIFAR-10 dataset~\cite{Krizhevsky09}). Thus, the null hypothesis $H_0$ formally means that $x \in D$, while the alternative hypothesis $H_1$ states that $x \in U \setminus D$.

However, when setting up an MI attack in practice (e.g., for a \dataproof against a production model such as ChatGPT), it is less clear what $H_0$ and $H_1$ formally mean. Here, we do not know what $D$ is, or how exactly it was sampled from some broader data universe (that's the whole point!). 
What should the null hypothesis encompass here? Should it cover models trained on \emph{all} possible training sets that do not contain $x$? What about training sets that contain some $x' \approx x$? Or hypothetical training sets that contain $2^{100}$ samples?

It is thus clear that we need a more precise formalization of what the null hypothesis actually means, if we plan to reject it. We propose such a formalization here.
Assume that the model developer collects some dataset $D$. We assume no knowledge about $D$ or how it was collected. To set up our hypotheses, we now imagine two possible worlds.  
We suppose that $x \in D$ (i.e., the target data \emph{was} collected for training).\footnote{We could also assume the opposite, and then have the model developer flip a coin to \emph{add} $x$ to $D$.} Then, the model developer flips a fair coin $b$. If $b=0$, they remove $x$ from $D$ to get $D_0$. If $b=1$, they do not remove $x$ and we have $D_1 = D$. Next, they train the model 
$f \gets \texttt{Train}(D_b)$.
Now, our hypotheses are:
\begin{align*}
	H_0&: \text{The model $f$ was trained on $D_0$ (i.e., without $x$)} \\
	H_1&: \text{The model $f$ was trained on $D_1$ (i.e., with $x$)}
\end{align*}

Intuitively, we can think of these hypotheses as modeling a plausible real-world scenario where the data $x$ may get omitted from data collection, either inadvertently (e.g., due to some error in data crawling) or intentionally (e.g., due to an explicit filter that removes copyrighted data). 

\paragraph{How to reject the null hypothesis?}
As in any standard hypothesis test, our goal is to \emph{reject the null hypothesis}, i.e., to show that $H_0$ is unlikely to be true given the observed model.

To do this, we select some test statistic $T(f, x)$ (e.g., the model's loss on $x$), and then show that the obtained value is extremely unlikely under the null hypothesis.
We define a \emph{critical region} (or rejection region) $S$ so that 
we reject the null hypothesis if 
$T(f, x) \in S$. We typically select $S$ so that the false positive rate (i.e., the chance that we conclude that the model was trained on our data, even though it was not) 
is bounded by some small value $\alpha$.
But \emph{crucially}, this requires that we formally or empirically model the distribution of the test statistic under the null hypothesis.
Formally, the false positive rate is given by: 
\begin{equation}
\label{eq:fpr}
\text{FPR} = \Pr_{f \sim \texttt{Train}(D_0)}[T(f, x) \in S \mid H_0] \;,
\end{equation}
where the probability is taken over the randomness of training $f$ and of 
computing the test statistic $T$. 
Analogously, the test's true positive rate is given by:
\begin{equation}
\label{eq:tpr}
\text{TPR} = \Pr_{f \sim \texttt{Train}(D_1)}[T(f, x) \in S \mid H_1].
\end{equation}

To demonstrate convincingly that the model $f$ is trained on the data $x$, the data creator or owner thus needs to show evidence of a test statistic that allows for the null hypothesis to be rejected at a low FPR.

\paragraph{What is the test statistic?}
Modern MI attacks are typically \emph{black-box} (i.e., they only assume query-access to the model $f$), and employ a test statistic $T(f, x)$ that depends on the model's \emph{loss} on the target sample, $\ell(f, x)$.
Early attacks (e.g.,~\cite{yeom2018privacy}) compute a global threshold $\tau$ 
for all samples $x$, and simply reject the null hypothesis (i.e., predict that $x$ is a member)
if $\ell(f, x) < \tau$.

More recent attacks (e.g.,~\cite{carlini2022membership, ye2022enhanced}) 
use a \emph{likelihood ratio}~\cite{neyman1933ix} of the two hypotheses as the test statistic, i.e., 
\begin{equation}
    T(f, x) = \frac{\mathcal{L}(\ell(f, x) \mid H_0)}{\mathcal{L}(\ell(f, x) \mid H_1)} \;,
\end{equation}
where the likelihood $\mathcal{L}$ is taken over the randomness of training the model $f$.
These attacks then predict that $x$ is a member if $T(f, x) \geq \tau$ for some threshold $\tau$.
Different attacks vary in the methods used to estimate these likelihoods and to select the threshold $\tau$, but these details are not important here.

\paragraph{What is the false positive rate?}
If the data creator $C$ wants to use an MI attack to convince a third party (e.g., a judge) that their data $x$ was used to train a model $f$, they need to provide evidence that their attack is statistically sound.
Specifically, they need to show that the attack has a low false positive rate: the probability that they would falsely accuse the model developer of using their data is very low.
So they would need to show that the quantity in~\Cref{eq:fpr}, the FPR, is small.


Again, when researchers use MI attacks ``in the lab'', where the training set is small and the process for creating it is fully known (e.g., by subsampling the CIFAR-10 dataset), then evaluating this probability is feasible. We can instantiate the hypothetical world modeled by $H_0$ and train multiple target models $f$ to estimate the probability in~\Cref{eq:fpr}.

However, when using an MI attack for a \dataproof against a large production model (like ChatGPT, or DALL-E), we claim that estimating this probability is \emph{impossible!} Indeed, we do not know what the training data $D$ is or how it was sampled (this is why we want a \dataproof in the first place). And even if we did know all training samples (with the exception of the bit $b$ determining whether the data $x$ was used or not), and the exact training procedure $\texttt{Train}$ used for the model $f$ (which, in most cases, we do not, as the training procedures and architectures of modern large-scale models are often proprietary), then sampling from the null hypothesis would still be prohibitively expensive as it would require training new large models.

Many existing works, which are motivated by proving membership in real-world production settings~\cite{meeus2024did, meeus2024inherent, shi2023detecting, eichler2024nob, duarte2024cop, mainidi2024, zhang2024min}, aim to approximate the FPR
in~\Cref{eq:fpr} without the need to explicitly sample from the null hypothesis (i.e., by retraining models $f$).
We show in \Cref{sec:case_studies} that all these methods are unsound.

\paragraph{An alternative test: Ranking counterfactuals.}
We now propose an alternative formalization of hypothesis testing that is sound and tractable (under suitable assumptions).
Here, we make an additional assumption that the data creator originally samples the target data $x$ randomly from some set $\mathcal{X}$, before $x$ is collected by the model developer.
For example, a news provider might add a uniformly random 9-digit identifier to each news article they post online.

A natural test statistic---originally proposed by~\citet{carlini2019secret}---is to \emph{rank} all samples in $\mathcal{X}$ according to the model $f$, and output the relative rank $\texttt{rank}(f, x, \mathcal{X})$ of the true target $x$.
For example, we can compute the collection of losses $\{\ell(f, x')\}$ for all $x' \in \mathcal{X}$, sort them, and then check the position of the target's loss
$\ell(f, x)$. If the target's loss is very low in this list, it is likely that the model was trained on it, rather than on any of the alternatives.
In our above example with random identifiers assigned to news articles, we would compute the model's loss on the sampled identifier (say 456-844-200), and then rank it among the losses on all other possible identifiers.\footnote{If the space $\mathcal{X}$ is too large to enumerate fully, we can estimate this rank via random sampling from  $\mathcal{X}$, as done in~\cite{carlini2019secret}.}

We then reject the null hypothesis if:
\begin{equation}
    \label{eq:rank_test}
    \texttt{rank}(f, x, \mathcal{X}) \leq k\;,
\end{equation}
for some constant $k \geq 1$.
The false positive rate of this membership inference attack is given by: 
\begin{equation}
\label{eq:fpr_rank}
\text{FPR} = \Pr_{\substack{x \sim \mathcal{X} \\ f \sim \texttt{Train}(D)}}[\texttt{rank}(f, x, \mathcal{X}) \leq k \mid H_0]
\end{equation}

Note that we are still taking probabilities over the randomness of the model training. So it would seem that we have not gained anything, and we still need 
to train new models to estimate this error rate. However, if the model $f$ is \emph{independent} of the random variable $x$ (note that we condition on $H_0$, i.e., on $f$ \emph{not} being trained on $x$), the FPR takes a very simple closed form:

\begin{lemma}
\label{lemma:fpr}
Assume that $x$ is sampled uniformly at random from $\mathcal{X}$ \emph{independently} of the creation of the training set $D$ and model training $f \sim \texttt{Train}(D)$.
Then, the FPR in \Cref{eq:fpr_rank} satisfies:
\begin{equation}
    \emph{FPR} \leq k / |\mathcal{X}| \;.
\end{equation}
\end{lemma}

\begin{proof}
    We can rewrite \Cref{eq:fpr_rank} as
    \begin{align*}
        \text{FPR} &= \Pr_{f \gets \texttt{Train}(D)}\ \Pr_{x \sim \mathcal{X}} [\texttt{rank}(f, x, \mathcal{X}) \leq k \mid H_0]\\
        &\leq \Pr_{f \gets \texttt{Train}(D)} [k / |\mathcal{X}|]\\
        &= k / |\mathcal{X}| \;.
    \end{align*}
\end{proof}

The inequality follows from the fact that for any \emph{fixed} $f$, at most $k$ samples in $\mathcal{X}$ can have rank at most $k$ (it can also be fewer in case of ties). We are then left with a probability over $f$ of a constant.

To apply this lemma and get a sound computation of the FPR, we need the assumptions in the lemma to be met.
That is, $x$ has to be sampled uniformly at random, \emph{and} independently of the model $f$. In \Cref{sec:solutions} we argue that these assumptions hold in three special cases:
(1) when we apply membership inference to specially designed random \emph{canaries}; (2) when we detect the watermarked data and (3) when we show a stronger \emph{data extraction} attack rather than mere membership inference.

However, as we will show in the next section, existing \dataproofs proposed in the literature fail to satisfy one or both of these assumptions.
Thus, these \dataproofs are unsound.

\section{Statistically Unsound \Dataproofs}
\label{sec:case_studies}

Since it is impossible to directly determine the true FPR defined in~\Cref{eq:fpr}, existing works employ various strategies to estimate an approximate FPR value. In this section, we review three main paradigms used for estimating this value, discuss the key issues associated with each approach, and provide case studies to illustrate these issues.

\subsection{Collecting Non-member Data a Posteriori}
\label{ssec:case_studies_1}
This is the most natural and widely proposed method for estimating the FPR of a membership 
inference attack on foundation models (e.g.,~\cite{meeus2024did, shi2023detecting, zhang2024min, duarte2024cop, duan2024membership}).
The idea is to take a set of data that is \emph{known} to not be in the model's training set, and then evaluate the test statistic on this set to approximate the null hypothesis.

For example, if we know that a foundation model's training set was collected before some ``cutoff'' date $T$, and the model has not been updated since, then any new data produced after time $T$ is necessarily a non-member of the training set. Alternatively, if a party has access to some held-out data that was not available to the model developer (e.g., data that is not publicly available on the Web), then this data could also safely be considered to consist of non-members.

Let's denote this set of known non-members as $\mathcal{X}$. Then, various prior works 
(e.g., \cite{meeus2024did, shi2023detecting, zhang2024min, duarte2024cop}) compute the FPR of an MI attack as:
\begin{equation}
\label{eq:fpr_aposteriori}
\Pr_{x' \in \mathcal{X}} [T(f, x') \in S] \;,
\end{equation}
Note that estimating this probability is easy, as it just requires evaluating the trained model $f$ on samples drawn from $\mathcal{X}$.
Thus, it would be easy to convince a third party that this value is small. But this \dataproof would be unsound, as we argue below.


\paragraph{Issue \#1: The ``bad'' model.}
The FPR computed in \Cref{eq:fpr_aposteriori} differs from the original goal in~\Cref{eq:fpr} in a crucial way.
Instead of taking the probability over the randomness of training $f$ (and evaluating on a fixed target $x$), we now instead take the probability over a set of different samples $\mathcal{X}$.
It is not at all clear how the quantity in \Cref{eq:fpr_aposteriori} should relate to the one we want in \Cref{eq:fpr}. In particular, maybe the model developer got unlucky and the actual model $f$ that was trained is simply a ``bad'' model that behaves unusually on the target $x$, but not on the other non-members in $\mathcal{X}$, as illustrated below.

\begin{figure}[h]
\begin{tcolorbox}[colback=gray!5!white, colframe=gray!80!black]
Suppose that when not trained on $x$, the  model's loss on $x$ follows a Bernoulli distribution: 
\[
\ell(f, x) = \begin{cases}
0 & \text{with probabilitiy $1/2$}\\
1 & \text{with probabilitiy $1/2$}\\
\end{cases}
\]
For all other samples $x' \in \mathcal{X}$ that are not in the training set, we have $\ell(f, x') = 1$.\\[-0.5em]

Suppose we reject the null hypothesis that $x$ was not used in training if we observe that $\ell(f, x) < 0.5$.
This test has a high FPR (according to \Cref{eq:fpr}) of $1/2$.
Yet, the ``approximation'' in \Cref{eq:fpr_aposteriori} yields $0$.
\end{tcolorbox}
\end{figure}

This issue is inherent to trying to directly approximate the FPR in \Cref{eq:fpr}.
Since we only have access to one \emph{fixed} model $f$, there is no way to estimate this probability.

Thus, going forward, we will consider that we instead aim to use the rank-based hypothesis test from \Cref{eq:rank_test}, and then aim to apply \Cref{lemma:fpr} to bound the FPR
(although this is not explicitly proposed in any of the works we survey, to the best of our knowledge).
{We note that using the rank-based test solved the issue of the ``bad'' model above: we might incur a false positive if we are unlucky and sample the counterfactual $x$ on which the model behaves unusually, but our estimate of the false positive \emph{rate} (which is a probability taken over the choice of counterfactual and model training) given by \Cref{lemma:fpr} is correct. See \Cref{apx:bad_model} for more details.}

So if the assumptions of \Cref{lemma:fpr} are met, we could bound the FPR without needing to sample from the null hypothesis. Unfortunately, we show that the setup of existing \dataproofs fails to meet these assumptions.

\paragraph{Issue \#2: Distribution shifts.} 
Non-members collected a posteriori 
{from} the target $x$ {are} unlikely to look like a random sample from $\mathcal{X}$, {as has been previously demonstrated in~\cite{das2024blind}}. 
{We therefore cannot apply \Cref{lemma:fpr} directly, since the distribution shifts mean} that we are not actually bounding the probability in \Cref{eq:fpr_rank} correctly. And indeed, the rank of $x$ could be very low compared to other samples in $\mathcal{X}$, even if the model $f$ was not trained on $x$---simply because the model might distinguish $x$ from other samples in $\mathcal{X}$.

\paragraph{Case study: Time-shifted articles.}
\begin{figure}[t]
\centering
\includegraphics[width=\linewidth]{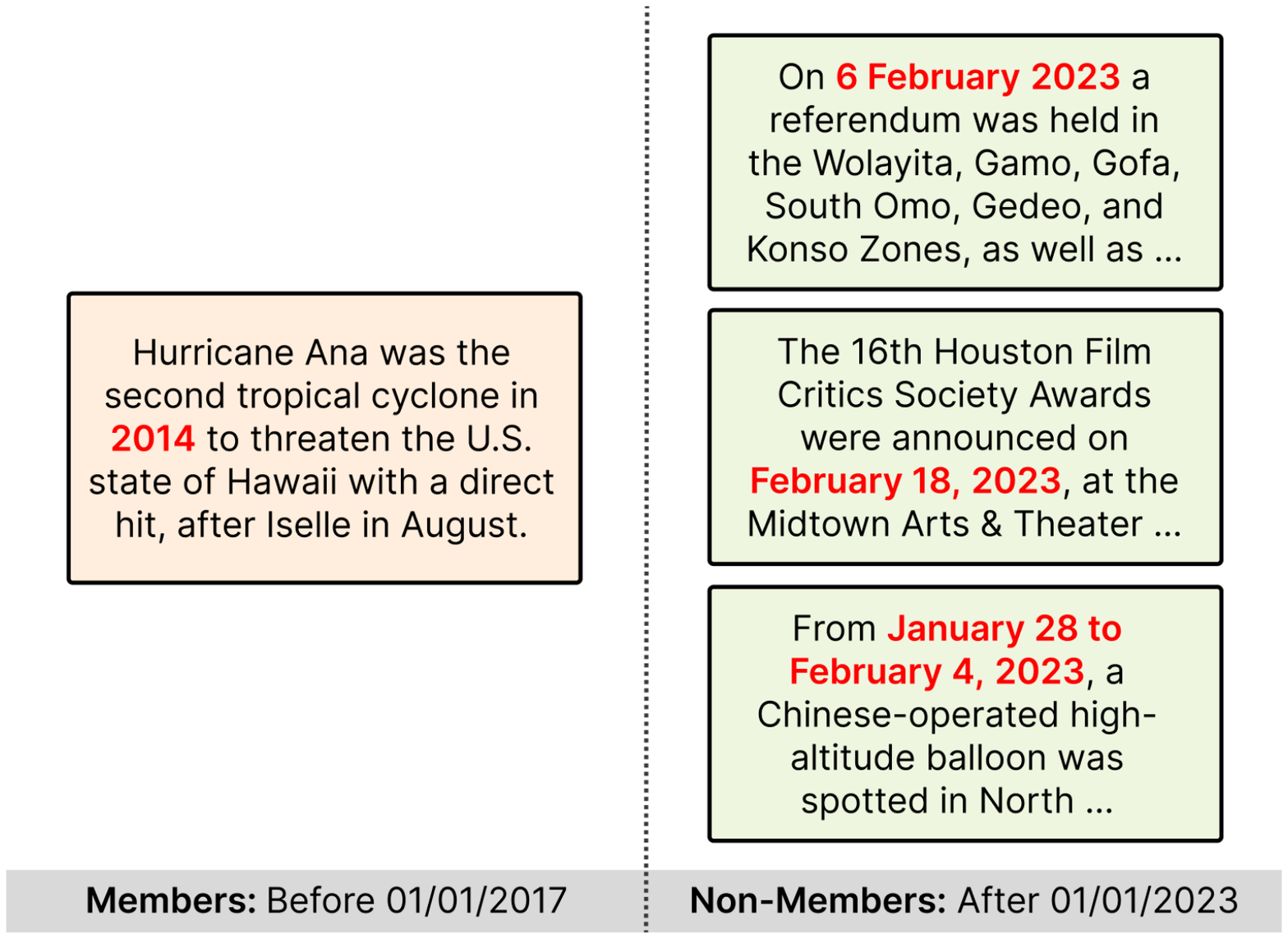}
\caption{If we try to estimate the FPR of a \dataproof by collecting non-members after the model's cutoff date (e.g., as in~\cite{shi2023detecting, meeus2024did}), we estimate the model's behavior on a distribution that differs significantly from the true null hypothesis.}
\label{fig:case_1}
\end{figure}

A common approach in the literature for collecting known non-member data is to take a set 
of samples published after the known \emph{cutoff date} $T$ of the model's training data~\cite{shi2023detecting, meeus2024did, duan2024membership, das2024blind}.
For example, if a news provider wants to show that some article published before time $T$ was used for training a model, they would collect other articles $\mathcal{X}$ published after time $T$ and use these to evaluate the FPR (see~\Cref{fig:case_1}).

Yet, {as shown in~\cite{das2024blind}}, these non-member articles {are} easily distinguishable from the target article $x$ due to distribution shifts. For example, articles in $\mathcal{X}$ might refer to events from after time $T$, which is unlikely to be the case for the target $x$. 
Or maybe the newer articles in $\mathcal{X}$ are on different topics or a slightly different style than $x$. All this means that we cannot directly judge if the model's behavior on $x$ is ``unusual'' just by comparing it to the model's behavior on $\mathcal{X}$.

This issue of distribution shifts in membership inference evaluations has been highlighted in multiple recent works~\cite{duan2024membership, das2024blind, meeus2024inherent, mainidi2024}.
{These works show that many \emph{current} benchmarks for MI attacks suffer from distribution shifts, and that naive collection of non-member data a posteriori is likely to lead to this issue. As a result, for current benchmarks, these works show that one can easily distinguish purported members (i.e., the target article $x$ in our example) from known non-members $\mathcal{X}$ even when $x$ was not used during training. These works also show that current MI attacks are largely ineffective when these distribution shifts are removed from the evaluation data.}

We {push this argument further and claim}  that this issue does not just mean that we have to be careful when evaluating MI attacks in academic settings, {and with current benchmarks and attacks.}
{Rather, we argue} that it is actually impossible to evaluate how well {\emph{any MI}} attack ({including} future, better ones) works in practice, {if we cannot properly sample from the null hypothesis.}

\subsection{Collecting Indistinguishable Non-members} 
\label{ssec:case_studies_2}

To mitigate temporal or other data distribution shifts highlighted above, some works construct a ``de-biased'' set of non-members~\cite{eichler2024nob, meeus2024inherent}.  That is, they ensure that the target $x$ is chosen such that we can (a posteriori)  build a set of non-member samples $\mathcal{X}$ 
so that $x$ looks like a random sample from $\mathcal{X}$.

As an example, Meeus et al.~\cite{meeus2024inherent} suggest to pick the target $x$ and the corresponding known non-members $\mathcal{X}$ from data published respectively right before and right after the target model's training \emph{cutoff date}.
Consider GPT-4, which was trained on
data collected before and up to April 2023 (allegedly). This means we can select possible members and known non-members from right before and after this date.

As we argued in \Cref{ssec:case_studies_1}, this will still not resolve the fundamental issues with trying to approximate the FPR in \Cref{eq:fpr} directly, as we cannot sample from the null hypothesis.
Thus, instead, we consider the rank-based test in \Cref{eq:rank_test} and aim to apply \Cref{lemma:fpr} to bound the FPR.


While debiasing the non-member distribution can address the issue of distribution shifts described above, other issues remain that may invalidate the assumption of \Cref{lemma:fpr}.

\paragraph{Issue \#1: Limited scope for \dataproofs.} 
This approach can only show that the MI attack has low FPR for specific targets $x$.
For instance, in our example above with GPT-4, we could only hope to provide a \dataproof for data published right before April 2023.
This approach does not provide any guarantees for samples further in the past or for data distributions with more complex temporal dynamics.
    
\paragraph{Issue \#2: Uncertainty around known non-members.} 
When selecting non-members
$\mathcal{X}$ to be as similar as possible to the target $x$, it may be hard to ensure that the samples from $\mathcal{X}$ are indeed non-members.
For example, the cutoff date reported by a model developer might not be fully accurate~\cite{cheng2024dated}, as the model might be periodically updated without changing the official cutoff date.

\paragraph{Issue \#3: Low test power.} 
Beyond the issues above that can cause an unsound approximation of the attack's FPR, focusing the attack on a target $x$ collected very close to a model's cutoff date could also make the attack very \emph{weak} (i.e., with low TPR).

First, it is unlikely that \emph{all} the training data was collected right before the cutoff date. Some training data sources might have been collected much further in the past~\cite{cheng2024dated} (e.g., OpenAI might have scraped the New York Times website in 2022 when collecting GPT-4's training data). And so it is possible that we would focus on creating \dataproofs for samples that are not members, even if the training set does contain older samples from the same source.

Second, even if the cutoff date is precise, we should expect that the samples that are most likely to be memorized by a model would have been published on the Web long before the cutoff date (as this makes it more likely for the data to be duplicated across the training set~\cite{carlini2022quantifying}).

\paragraph{Case study: ArXiv articles published right after an LLM's cutoff date.}
\begin{figure}[t]
    \centering
    \includegraphics[width=\linewidth]{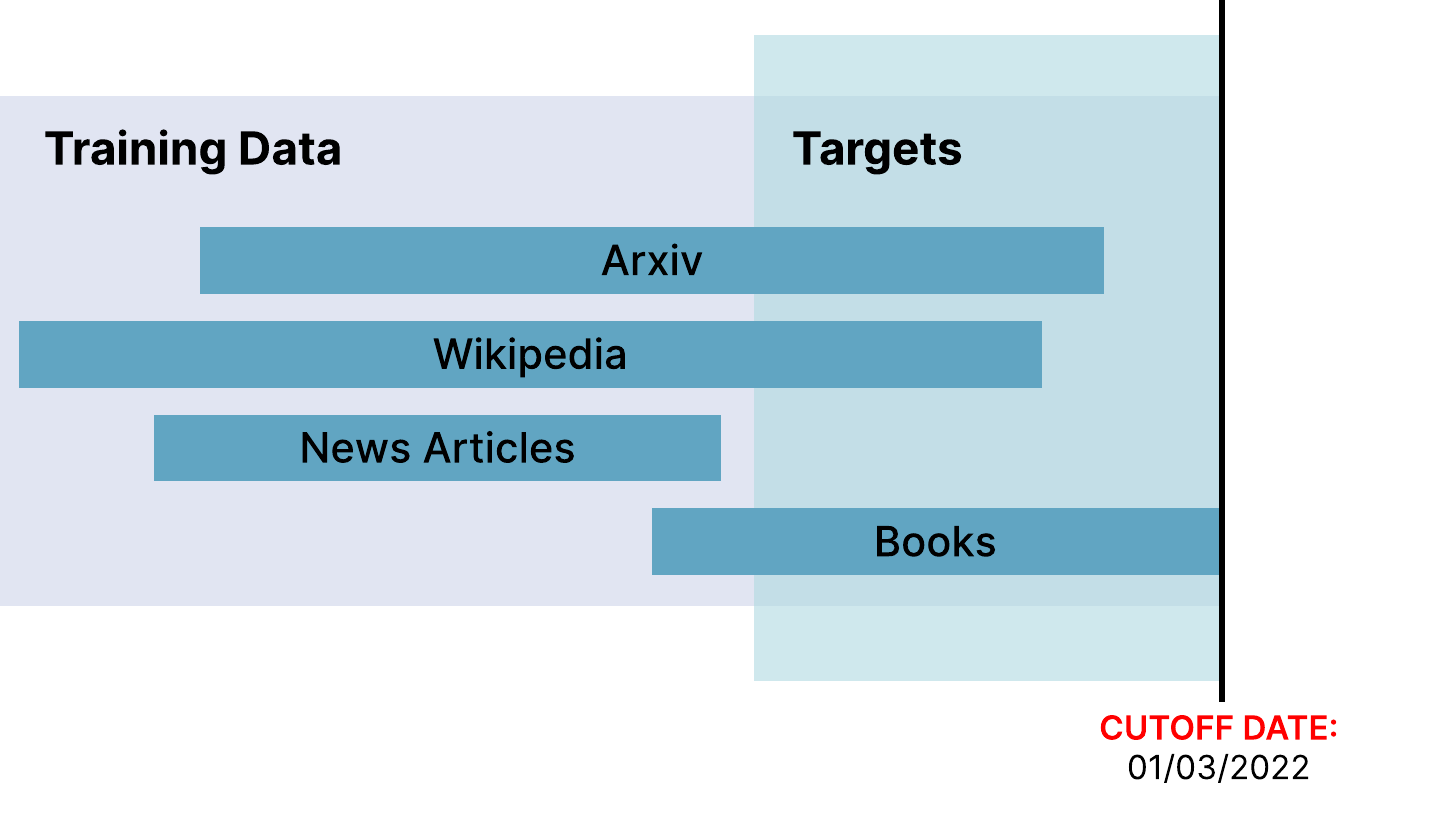}
    \caption{If we collect targets for \dataproofs close to the model's cutoff date (as suggested in~\cite{meeus2024inherent}), we run the risk of focusing our efforts on targets that have no chance of being members, because some data sources could have been collected further in the past.}
    \label{fig:case_2}
\end{figure}
Recent efforts to mitigate distribution shifts involve creating ``non-biased'' and ``non-classifiable'' data post hoc~\cite{eichler2024nob} or collecting non-members within a small time window~\cite{meeus2024inherent}. For instance, Meeus et al.~\cite{meeus2024inherent} sample arXiv papers from one month before the cutoff date of OpenLLaMA as target members and papers released one month after as non-members. They show that simple models fail to distinguish the targets and known non-members, thereby alleviating the concerns about distributions shifts.

However, this approach has significant limitations. First, this approach only works if we want to provide \dataproofs for data sources that are published regularly, and with little distribution shift over time.
But even then, we may get a very weak attack. Consider the OpenLLaMA model used in~\cite{meeus2024inherent}, which has a cutoff date of February 2023. All this tells us is that the \emph{latest} data that was collected was from February 2023. But maybe some of the data sources (e.g., arXiv) could have been collected much earlier than that, as illustrated in~\Cref{fig:case_2}. Then, all the targets considered by the data creator would not be in the training set, even if their earlier data was used. But this earlier data is harder to use for a \dataproof, as known non-members (which are from after the data cutoff date) might then have a significant distribution shift.

Finally, to make matters worse, some model developers like OpenAI acknowledge that both pre-training and post-training datasets may include data from after the ``official'' data cutoff~\cite{achiam2023gpt}. Thus, the boundaries between training and non-training data are blurred, further complicating time-based methods for constructing members and non-members.

Note that researchers can evaluate membership inference attacks using known independent and identically distributed (IID) members and non-members. This can be done by either testing on a fixed model trained with known train/test split~\cite{duan2024membership}, such as the Pythia models trained on Pile~\cite{biderman2023pythia}, or by fine-tuning a fixed model with a known dataset~\cite{duan2023diffusion}. However, this only demonstrates that the MI attack has a low false positive rate for those specific models, datasets, and targets. It does not offer insights into the performance of production models, as it still does not sample from the null hypothesis and therefore cannot serve as sound \dataproofs.

\subsection{Holding Out Counterfactual Data}
\label{ssec:case_studies_3}
The issue with the two previous approaches we considered is that we tried to approximate the attack's FPR by choosing a new distribution of non-members \emph{a posteriori}.
An alternative would be to target data that was explicitly sampled from some distribution $\mathcal{X}$ before being (possibly) included in the training data. 
Then, other samples in $\mathcal{X}$ can be \emph{held-out} to be later used as non-members.

This is the approach taken by Maini et al.~\cite{mainidi2024}. As an example, they suggest that an author who wants to create a \dataproof for some piece of writing $x$ (e.g., a news article from the New York Times) could claim that the data $x$ that was published is a random sample of earlier ``drafts'' $\mathcal{X}$ of that data. 

To give evidence that the data $x$ was used for training, the author would then show that the test statistic on the published sample $x$ is highly unusual compared to the test statistics obtained for \emph{counterfactual} samples $x' \sim \mathcal{X}$ that are not part of the model's training set, but which are a priori equally likely to have been used for training.
For example, the New York Times might show that a trained LLM has much lower loss on a real published article than on earlier drafts of that article that were never published.

This is exactly the rank-test approach we formalize in \Cref{eq:rank_test},
but evaluated a posteriori for a specific sampled target $x$.
We could thus hope to apply \Cref{lemma:fpr} to bound the FPR.


As we will argue later in \Cref{sec:solutions}, this approach can be made sound with a judicious choice of target $x$ and set $\mathcal{X}$. However, in general this approach still suffers from severe issues.

\paragraph{Issue \#1: Non-IID sampling.} 
In general, it is difficult to ensure that the published sample $x$ and the held-out samples $x' \sim \mathcal{X}$ are truly identically distributed. 
In the example above for news article drafts, it is natural to consider that the final version of an article would have a slightly different quality, completeness, and structure
than earlier drafts.
Thus, the model $f$ might rank $x$ higher than alternatives even when not trained on $x$.
    
\paragraph{Issue \#2: Causal effects of publishing.} A much more significant issue that is overlooked by Maini et al.~\cite{mainidi2024} is that the act of publishing a sample $x$ (e.g., on the Web) may influence \emph{other samples} in the model developer's training set $D$.
For example, the final article published by a news provider might get quoted in other websites or inspire other articles. And so, even if this specific article is not in the training data, the model might still be trained on other samples that were influenced by it.

To formalize this, we introduce the concept of a \emph{data repository} $\hat{\mathcal{D}}$, a public repository of possible data samples that a model developer might collect (we can think of $\hat{\mathcal{D}}$ as being the public internet). So when the data creator publishes $x$, they add it to the data repository $\hat{\mathcal{D}}$.
With this in mind, we can write out the FPR from \Cref{eq:fpr_rank} that we want to bound in more detail as follows:

{\footnotesize{
\begin{equation}
    \label{eq:fpr_rank_detail}
    \text{FPR} = 
    \hspace{-3pt}
    \Pr_{x \sim \mathcal{X}}
    \hspace{-3pt}
    \left[ \texttt{rank}(f, x, \mathcal{X}) \leq k
        \ \middle|
        \hspace{-3pt}
        \begin{array}{r}
            
            \text{$C$ publishes $x$ in $\hat{\mathcal{D}}$ at time $t$} \\
            \text{$M$ collects $D \subset \hat{\mathcal{D}}$ at time $T > t$} \\
            D_0 \gets D \setminus \{x\}\\
            f \gets \texttt{Train}(D_0)
        \end{array}
        \hspace{-4pt}
        \right]
        \hspace{-3pt}
\end{equation}}
}

This makes it easier to see that the bound in \Cref{lemma:fpr} no longer applies here, when the act of publishing the sample $x$ causally influences other samples in the data repository $\hat{\mathcal{D}}$ between times $t$ and $T$.
Indeed, the bound in \Cref{lemma:fpr} crucially assumes that $f$ is independent of the choice of $x$, which allows us to average over the randomness of the training run.
But if the model training depends on the choice of target $x$ (even if it is not trained on it explicitly), the FPR calculation becomes much more complex: for each counterfactual $x'$, we would need to simulate the causal effect that $x'$ would have on the rest of $\hat{\mathcal{D}}$---and thus the training set $D$---and retrain a model $f$ accordingly.

\paragraph{Case study: Held-out writing drafts of Harry Potter.}
For sake of argument, let's assume that when writing the first Harry Potter book, author J.K. Rowling sampled the protagonist's name uniformly at random from some list. Thus, a priori, 
``Harry Potter and the Philosopher's Stone'' was equally likely to be published as 
``Leo Osborne and the Philosopher's Stone.''\footnote{According to UK statistics, Leo and Osborne are equally common first names and last names as Harry and Potter, respectively.}
This is an ideal setting to implement the \dataproof proposed by Maini et al.~\cite{mainidi2024}.

Now, suppose that some model $f$ was \emph{not} trained on the Harry Potter books. That is, after collecting a large training set from the Web, the model developers explicitly removed all excerpts from the books (to make the argument stronger, we can even assume the model developers only trained on permissively licensed internet content, such as Wikipedia).
Our null hypothesis is thus true here: the model was \emph{not} trained on Harry Potter.
Yet, we claim that this model is very likely to assign lower loss to the actual published Harry Potter book contents, compared to the counterfactual where all mentions of ``Harry Potter'' are replaced by ``Leo Osborne.''

The reason is simple: there is a wealth of additional data \emph{about} Harry Potter on the Web, even if it is not direct excerpts from the book. For example, there are various Wikipedia pages about Harry Potter, fan content, news pieces, reviews, etc.
Thus, even if a model was not trained on any of the original text of the Harry Potter books, it is still likely to learn about the concept of \emph{Harry Potter} books, rather than, say,
``Leo Osborne books.''

\begin{figure}
    \centering
    \includegraphics[width=\linewidth]{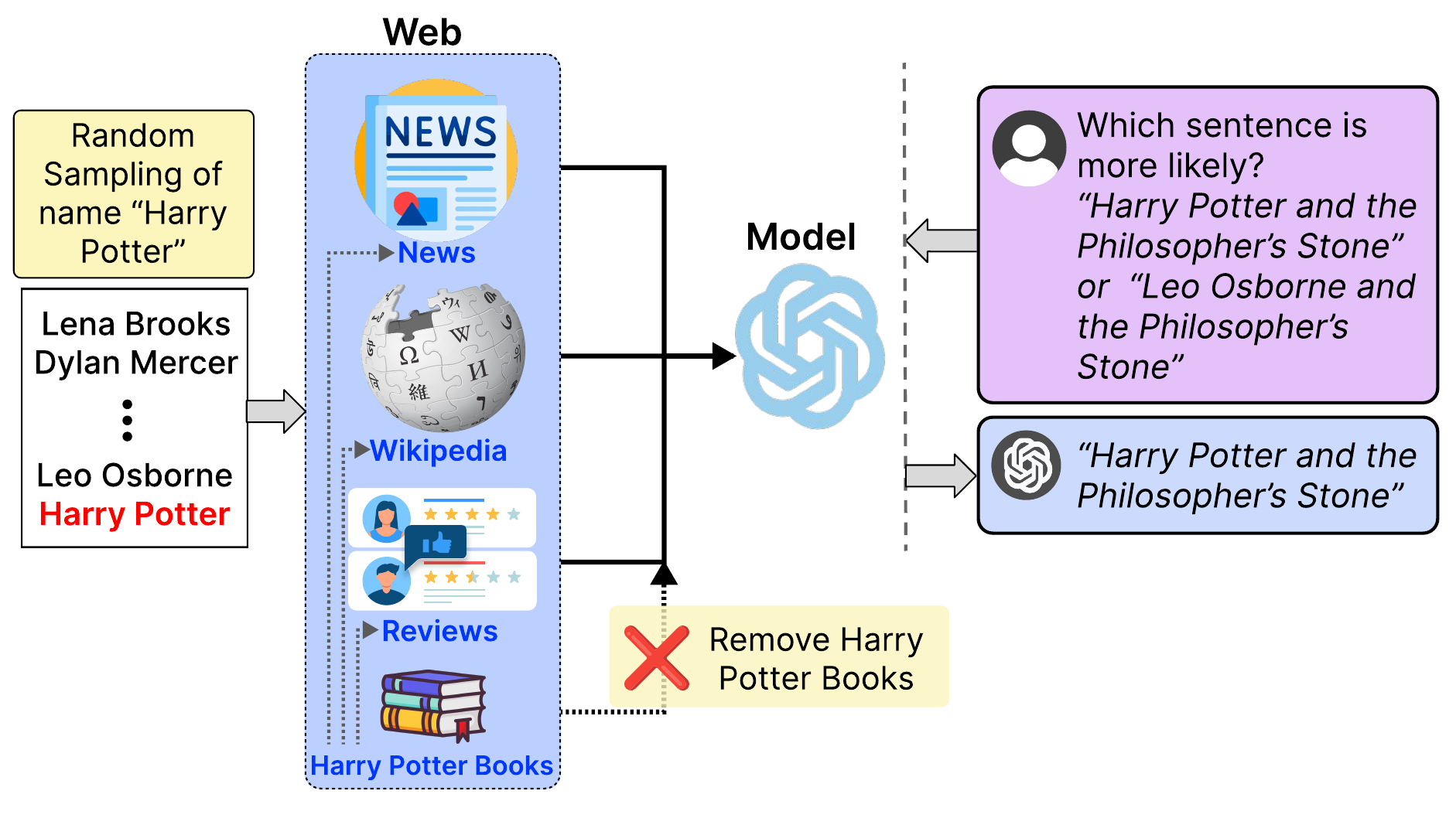} 
    \caption{If we want to perform a \dataproof by comparing the model's behavior on a piece of data to plausible \emph{counterfactuals}~\cite{mainidi2024} (in this fictitious case, possible alternative names that J.K. Rowling could have sampled in lieu of ``Harry Potter''), we need to ensure that the act of \emph{publishing} the data did not causally impact other parts of the model's training data (that are not part of the claimed \dataproof).}
    \label{fig:case_3}
\end{figure}

Thus, when using our rank-based hypothesis test, we would overestimate how ``unlikely'' the model's familiarity with the text of Harry Potter is. Indeed, we are not just measuring whether the actual book contents are known to the model, but whether the entire concept of Harry Potter is more likely to the model than any of the a priori alternatives.

As illustrated in~\Cref{fig:case_3}, the right way to evaluate an MI attack's FPR in this setting would require re-simulating, for each counterfactual $x'$, the entire causal process of publishing $x'$ in lieu of $x$.
Concretely, we would have to find every mention of Harry Potter on the Web and replace it with 
Leo Osborne. We would also have to adapt all popular culture that is derived from Harry Potter
(e.g., its influence on the popularity of the name ``Harry'' for newborns, and so on).
Then, we would need to retrain the model $f$ and evaluate the attack on it.

This example underscores a broader challenge in performing proper membership inference on large foundation models. Since it is hard to fully isolate a specific data point or concept once it has entered the public domain, the notion of ``not being in the training set'' can become blurred~\cite{brown2022does}.
As a concrete example, some authors sued OpenAI for allegedly training ChatGPT on their books, and used as evidence the fact that ChatGPT can write a summary of the book.\footnote{
\url{https://www.theverge.com/2023/7/9/23788741}.
}
Yet, this behavior may be consistent with a model that was never trained on the actual contents of these books, and merely on many other pieces of Web data that refer to the book, e.g., the book's Wikipedia page.
In this sense, presenting a sound argument of membership in the training set is challenging.

\begin{figure}[t]
\begin{tcolorbox}[colback=gray!5!white, colframe=gray!80!black, title= Injecting Random Canaries]
Last night, a meteor shower lit up the skies over Europe, drawing crowds of stargazers. Experts say the meteor shower was the brightest in the past decade, with over 500 visible meteors in just an hour. In other news, tech giant HyperTech unveiled a new smartphone with a holographic display, sparking excitement among consumers. Meanwhile, global markets are reacting to rising tensions between two major economic powers, which could lead to significant fluctuations in the coming weeks. Finally, an ancient shipwreck was discovered off the coast of Greece, believed to date back to the Roman Empire.\\
{\color{blue}$<$div style=``display: none''$>$\textbf{vR8xJ4tG-603d}$<$div$>$}
\end{tcolorbox}
\caption{Injecting a specially crafted \emph{canary} into a news article, e.g., a hidden message in the HTML code.}
\label{tbox:canaries}
\end{figure}

\section{Sound \Dataproofs}
\label{sec:solutions}
We now show that three alternative paradigms can yield convincing and sound \dataproofs: (1) membership inference attacks on specially crafted \emph{canaries};
(2) watermarked training data, and (3) data extraction, where non-trivial portions of the training data are directly recovered from the model.

\subsection{Injecting Random Canaries}
Recall that to apply \Cref{lemma:fpr} to our rank-based hypothesis test, we need to assume that the target data $x$ is sampled from some candidate set $\mathcal{X}$ \emph{independently} of the model $f$. The remainder of the set can then be used as counterfactuals for the hypothesis test. As we argued in \Cref{ssec:case_studies_3}, issues arise if: (1) the sampling is not IID; and (2) if the published data $x$ influences other samples in the training set.

The first part is easy to address, by ensuring that the target data is explicitly sampled uniformly at random from a candidate set. That is, instead of taking some ``real'' data and then a posteriori trying to build a superset that it could have been sampled from, we explicitly sample and inject a piece of ``fake'' data in the data stream we want to test.
This is the notion of a ``data canary''~\cite{carlini2019secret}, which is commonly used for privacy auditing~\cite{JagielskiUO20, NasrSTPC21, steinke2023privacy}.

Addressing the second issue is more challenging. Here, we want to ensure that the injected canary carries no useful information, so that it does not impact other data published on the Web. If the canary string is some random-looking text that is unlikely to be read by anyone other than a data crawler, this condition is likely to be met. An example is illustrated in~\Cref{tbox:canaries}.

If these assumptions are met, then \Cref{lemma:fpr} applies and the FPR in \Cref{eq:fpr_rank} is simply equal to $k / |\mathcal{X}|$. So if we want to give a \dataproof with an FPR bounded by 1\%, we have to show some test statistic that ranks the true canary $x$ in the $1\%$-percentile of all possible canary strings $\mathcal{X}$.

The use of random canaries for \dataproofs was also proposed under the notion of ``copyright traps'' by Meeus et al.~\cite{meeus2024copyright}, although they do not consider a rank-based test and do not show how to bound the FPR of this approach.

\begin{figure}[t]
\begin{tcolorbox}[colback=gray!5!white, colframe=gray!80!black, title=Watermarking]

\textbf{Watermarked Data:} \\
{The sun was setting over the quiet village, casting long shadows across the cobblestone streets. 
Maria stood at the edge of the square, watching as the last market stalls were taken down. 
She knew she }{would have to speak to him soon, but fear gripped her heart. 
Just then, \ldots a figure appeared in the distance, moving steadily towards her. 
It was John, his face unreadable, but his eyes locked on hers.}

\vspace{0.5cm}
\hrule
\vspace{0.3cm}

\textbf{LLM Prompt:} \\
Complete the following text. \\
{``The sun was setting over the quiet village, casting long shadows across the cobblestone streets.
Maria stood at the edge of the square, watching as the last market stalls were taken down. She knew she would''}

\vspace{0.5cm}
\hrule
\vspace{0.3cm}

\textbf{LLM Output with Watermarks:} \\
\colorbox{green!30}{\strut have}%
\colorbox{red!30}{\strut to}%
\colorbox{green!30}{\strut speak}%
\colorbox{green!30}{\strut to}%
\colorbox{red!30}{\strut him}%
\colorbox{green!30}{\strut soon,}%
\colorbox{green!30}{\strut but}%
\colorbox{red!30}{\strut fear gripped}\\
\colorbox{green!30}{\strut  her heart.  Just then, \ldots a figure}%
\colorbox{red!30}{\strut appeared}%
\colorbox{green!30}{\strut in the}\\
\colorbox{green!30}{\strut distance,}%
\colorbox{green!30}{\strut moving}%
\colorbox{green!30}{\strut steadily}%
\colorbox{green!30}{\strut towards}%
\colorbox{green!30}{\strut her.}%
\colorbox{green!30}{\strut It}%
\colorbox{green!30}{\strut was John,}\\
\colorbox{green!30}{\strut his} 
\colorbox{red!30}{\strut face}%
\colorbox{green!30}{\strut unreadable, but}%
\colorbox{red!30}{\strut his}%
\colorbox{green!30}{\strut eyes}%
\colorbox{green!30}{\strut locked}%
\colorbox{green!30}{\strut on}%
\colorbox{red!30}{\strut hers.}
\end{tcolorbox}
\caption{A common method for identifying a watermark in the output of a language model is to count the ``green'' and ``red'' tokens and perform a statistical test that yields interpretable p-values~\cite{kirchenbauer2023watermark}.}
\label{tbox:watermark}
\end{figure}

\subsection{Watermarking}
Rather than relying on the formulation in~\Cref{eq:fpr} and \Cref{lemma:fpr} to bound the FPR, watermarked training data offers an alternative, statistically sound method to bound the FPR, as discussed in~\cite{kirchenbauer2023watermark}.
In this approach, the data creator embeds an imperceptible signal into the training data before its public release. This is done by using a randomly sampled seed to partition the vocabulary into a ``green list'' and a ``red list,'' subtly promoting the use of green tokens during sampling.
Essentially, we can view watermarked text as a way of embedding a random canary string---the watermark key---into a stream of text.

When a model is trained on this watermarked data, the signal leaves subtle but detectable traces in the model's learned patterns. Researchers can then formulate a hypothesis test to determine if the model was exposed to the watermarked data during training~\cite{sander2025watermarking,zhong2024watermarking}, and it can be detected with high confidence (low FPR) even when as little as 5\% of the training text is watermarked. 

As shown in~\Cref{tbox:watermark}, when we examine the generated text, we count the number of ``red'' and ``green'' tokens it contains. If a human had written it, we would expect to see around 9 green tokens. However, the text actually has 28 green tokens. The likelihood of this occurring by chance is incredibly small---$7\times10^{-14}$~\cite{kirchenbauer2023watermark}. This extremely low probability strongly indicates that the text was the watermarked data, which can serve as a sound \dataproof.

\subsection{Verbatim Data Extraction}
In existing lawsuits that have claimed non-permissive data usage for training, the accusing parties typically provide stronger evidence than mere membership inference, namely that they can \emph{extract} a portion of their data from a trained (generative) model~\cite{nyt_v_OpenAI, GettyImages_v_StabilityAI, wei2024evaluating}.
That is, they show the existence of some input (or prompt) $p$ within some set $\mathcal{P}$ such that $f(p) = x$ (or, more generally, an output that is very close to $x$ for some suitable metric).

Note that this is just another membership inference test: here, the test statistic $T(f, x)$ is of the form:
\[
    \mathbbm{1}\{\exists p \in \mathcal{P}: f(p) = x\} \;,
\]
i.e., a test of whether we can find a prompt that makes the model generate the target output.

What do we need for this hypothesis test to have low FPR?
Existing works on data extraction for LLMs provide some heuristics~\cite{carlini2021extracting, nasr2023scalable}, such as: (1) the output string $x$ should be ``long'' (e.g., 50 tokens, or 2-3 sentences); (2) $x$ should have ``high entropy''; (3) the prompts $p \in \mathcal{P}$ should be ``short'' or ``random looking''. Intuitively, these heuristics are meant to guard against the model $f$ emitting $x$ by chance alone, or by copying it from a prompt $p$.

\begin{figure}[t]
\begin{tcolorbox}[colback=gray!5!white, colframe=gray!80!black, title=Data Extraction From LLMs]

\textbf{Private Data:} \\
{\color{mygreen}The sun was setting over the quiet village, casting long shadows across the cobblestone streets. 
Maria stood at the edge of the square, watching as the last market stalls were taken down. 
She knew she }{\color{blue} would have to speak to him soon, but fear gripped her heart. 
Just then, a figure appeared in the distance, moving steadily towards her. 
It was John, his face unreadable, but his eyes locked on hers.}

\vspace{0.5cm}
\hrule
\vspace{0.3cm}

\textbf{LLM Prompt:} \\
Complete the following text. \\
{\color{mygreen}``The sun was setting over the quiet village, casting long shadows across the cobblestone streets.
Maria stood at the edge of the square, watching as the last market stalls were taken down. She knew she''}

\vspace{0.5cm}
\hrule
\vspace{0.3cm}

\textbf{LLM Output:} \\
{\color{blue}would have to speak to him soon, but fear gripped her heart. 
Just then, a figure appeared in the distance, moving steadily towards her. 
It was John, his face unreadable, but his eyes locked on hers.}
\end{tcolorbox}
\caption{Illustration of data extraction from a language model: Given a private text passage and a corresponding prompt, the model generates a continuation, highlighting the potential for verbatim data extraction.}
\label{tbox:extract}
\end{figure}

We now propose a way to (partially) formalize these intuitions by casting data extraction as a ranking hypothesis test as in \Cref{eq:rank_test}.
Given a set of prompts $\mathcal{P}$, define $\mathcal{U}$ as the \emph{universe of plausible outputs} (of some bounded length $l$) on prompts $p \in \mathcal{P}$. That is, $\mathcal{U}$ contains any string of length up to $l$ (presumably in addition to $x$) that would be a plausible continuation of some prompt $p \in \mathcal{P}$. We leave a more precise definition of $\mathcal{U}$ out of scope and instead provide a few examples for specific prompts:

\begin{itemize}
    \item If \emph{\color{mygreen}$p = $``1 2 3 4''}, then the outcome \emph{\color{blue}``5 6 7 8...''} is particularly likely, and no other strings should be expected to be as likely.
    \item If \emph{\color{mygreen}$p = $``repeat: 5cghdbnk4e7''}, then \emph{\color{blue}``5cghdbnk4e7''} is a particularly likely continuation, and no other strings should be expected to be very likely.
    \item If \emph{\color{mygreen}$p = $``WASHINGTON — American intelligence officials have concluded that a Russian military intelligence unit secretly offered bounties to Taliban-linked militants for killing coalition''}, then 
    the continuation \emph{\color{blue}``forces in Afghanistan — including targeting American troops — amid the peace talks to end the long running war there, according to officials briefed on the matter''} is certainly one plausible continuation.\footnote{This is the actual continuation of this prompt in a New York Times article memorized by ChatGPT~\cite{nyt_v_OpenAI}.} But one could imagine an uncountable number of alternatives. So here the universe $\mathcal{U}$ is presumably very large.
\end{itemize}

Why does this universe $\mathcal{U}$ matter? Our claim is that by testing whether $f(p) = x$ for some $p \in \mathcal{P}$, we are implicitly conducting the following counterfactual test:
\begin{equation}
    \exists p \in \mathcal{P} \text{ s.t. } \texttt{rank}(f, z, \mathcal{U}, x) = 1 \;,
\end{equation}
that is, we test whether the string $x$ is \emph{the most likely} output\footnote{If the sampling is randomized, we can generalize this to say that $x$ is a likely output of the model $f$.} of $f$ among all possible outputs in $\mathcal{U}$, for at least one prompt $p$.
%

Now, if we make the (simplifying) assumption that all strings in $\mathcal{U}$ are a priori approximately equally likely given a prompt $p$, then we can bound the FPR as:
\begin{equation}
    \text{FPR} \lessapprox |\mathcal{P}| / |\mathcal{U}| \;.
\end{equation}
That is, as long as we do not test too many different prompts (or too long ones), and the number of plausible outputs for these prompts is a priori high, then the likelihood that a model that was not trained on the target $x$ would output $x$ is very small. We show an example in~\Cref{tbox:extract}.

This formalism clarifies where existing heuristics for selecting targets and prompts for data extraction come from. By choosing short prompts, we minimize the size of $\mathcal{P}$; by choosing random prompts (or prompts that appear unrelated to $x$), we maximize the size of $\mathcal{U}$, i.e., the set of outputs that could plausibly follow the prompt; and by ensuring that the target string $x$ is long and of high enough entropy, we also ensure that many alternative plausible outputs should exist, thereby also maximizing the size of $\mathcal{U}$.

Now of course in practice we cannot easily characterize the exact sizes of $\mathcal{P}$ and $\mathcal{U}$.
The prevailing wisdom behind existing heuristics is that they result in a set $\mathcal{U}$ that is substantially larger than $\mathcal{P}$, and so the FPR of this test is essentially zero for all practical purposes. And indeed, existing works on data extraction attacks that followed these heuristics report no evidence of any false positives~\cite{carlini2021extracting, nasr2023scalable}.

\section{When is Membership Inference Sound?}
While we argue that membership inference attacks cannot provide evidence for a \dataproof in large-scale, production ML models, this does not mean that 
membership inference attacks are not useful and sound in other settings.
We review these here.

\paragraph{Empirical evaluations for privacy defenses.} 
In academic settings, membership inference attacks are typically used as privacy evaluation tools to assess privacy leakage in specific defenses~\cite{aerni2024evaluations, carlini2022free}, rather than as proofs for the training data itself. In these scenarios, the entire data generation process is transparent and reproducible, allowing researchers to know precisely how the model was trained. This transparency makes it feasible to replicate model training both with and without specific data points. As a result, researchers can directly sample from the null hypothesis (i.e., models trained without the target data), enabling accurate estimation of the attack's false positive rate, as defined in~\Cref{eq:fpr}.

\paragraph{Auditing and lower bounding differential privacy.}
In a similar vein, MI attacks can be used to audit implementations of differentially private machine learning~\cite{JagielskiUO20,NasrSTPC21, steinke2023privacy},
by drawing on a hypothesis testing interpretation of differential privacy~\cite{WassermanZ10,KairouzOV15}.
Here also, the assumption is that the auditor controls the data generation process, and can thus subsample training data to obtain sound FPR estimations.

\paragraph{A subcomponent of a data extraction attack.}
In a number of works on data extraction attacks~\cite{carlini2021extracting, carlini2023extracting}, an MI attack is applied at a filtering stage.
More precisely, these attacks first generate a very large number of 
candidate samples with a generative model, and then filter this set to a smaller size that can be checked for memorization with costly means (e.g., human verification). Here, we do not need to know precisely what the MI attack's FPR is. As long as the FPR is reasonably low, the costly verification work is minimized.

\section{Conclusion}
Our analysis highlights the fundamental limitations of using membership inference attacks as a reliable method for creating \dataproofs, due to the inability to  demonstrate that the attack has a low false positive rate (i.e., the chance of wrongly accusing a model developer). 
This issue is due to the difficulty in correctly sampling from the \emph{null hypothesis} (i.e., the model was \emph{not} trained on the target data), given the opaque and large-scale nature of modern foundation models.
Despite attempts to address these issues using debiased or counterfactual datasets, none of the proposed approaches we survey are statistically sound.

However, we show that three folklore approaches for generating \dataproofs
are sound under mild assumptions.
The first approach injects random canaries into potential training data, and then demonstrates that the model's behavior on the chosen canary differs significantly from the behavior on other canaries that were a priori equally likely.  
The second approach involves watermarking the training data, offering a statistically sound method by employing a statistical test that yields interpretable p-values.
The third is data extraction attacks, which involve retrieving identifiable portions of the training data from the model itself. This method avoids the complexities of direct counterfactual estimation and provides compelling evidence when large sections of the data can be extracted.
Both methods offer more robust and reliable means for demonstrating that a model was trained on specific data, and thus pave the way for stronger, more convincing \dataproofs.

\section*{Acknowledgment}
We thank all anonymous reviewers for their valuable comments and suggestions. J.Z. is funded by the Swiss National Science Foundation (SNSF) project grant 214838.
\balance

\newpage
\bibliography{reference}
\bibliographystyle{plainnat}
\balance
\newpage
\appendices

\section{{Solving the ``Bad'' Model Issue with the Rank-based Hypothesis Test}}
\label{apx:bad_model}
{
In \Cref{ssec:case_studies_1}, we illustrated how a naive estimation of the FPR of an MI attack on a set of non-members could fail if we were in the presence of a ``bad'' model, which acts unusually on the targeted point $x$.
Indeed, we showed that the true FPR here (over the probability of training the model $f$) is $1/2$, while the \emph{a posteriori} estimate in \eqref{eq:fpr_aposteriori} yields $0$ in this case.}

{
We now show that our alternative rank-based hypothesis test does not suffer from this issue.}

{
Note that since this test requires choosing the target $x$ \emph{a priori} from a set of counterfactuals $\mathcal{X}$, the true FPR of the test (given in \eqref{eq:fpr_rank}) is a probability taken over \emph{both} the training of the model and the random choice of the target.
\Cref{lemma:fpr} says that if we guess that $x$ is a member if its loss has rank $k$ among other samples in $\mathcal{X}$, the FPR is $k/|\mathcal{X}|$ (note that this estimate is independent of the model $f$!).
This remains true for the ``bad'' model example:}

{
\begin{align*}
    \text{FPR} =& \Pr_{\substack{z \sim \mathcal{X} \\ f \gets \texttt{Train}(D)}}[\texttt{rank}(f, z, \mathcal{X}) \leq k \mid H_0] \\
    =& \frac{1}{|\mathcal{X}|} \Pr[\texttt{rank}(f, x, \mathcal{X}) \leq k \mid x \notin D] 
    \\&+ \frac{|\mathcal{X}|-1}{|\mathcal{X}|} \Pr[\texttt{rank}(f, x', \mathcal{X}) \leq k \mid x' \notin D] \\
=& \frac{1}{|\mathcal{X}|} \left(\frac12 + \frac12 \frac{k}{|\mathcal{X}|}\right) + \frac{|\mathcal{X}|-1}{|\mathcal{X}|} \left(\frac12 \frac{k-1}{|\mathcal{X}|-1} + \frac12 \frac{k}{|\mathcal{X}|}\right) \\
= & \frac{k}{|\mathcal{X}|} \;.
\end{align*}
}

{Of course, if we happen to sample the target $x$ for which we get a ``bad'' model, we will incur a false positive with 50\% probability.
But this event is rare \emph{a priori} and so our 
computation of the FPR remains correct.}

\end{document}